\documentclass[letterpaper, 10 pt, conference]{ieeeconf}
\IEEEoverridecommandlockouts
\usepackage{cite}
\usepackage{amsmath,amssymb,amsfonts}
\usepackage{tikz}
\usetikzlibrary{positioning,calc, shapes, fit}
\tikzset{container/.style={draw, rectangle, dashed, inner sep=2em}}

\usepackage{varwidth}
\usepackage{ragged2e}
\usepackage{graphicx}
\usepackage{epstopdf}
\usepackage{textcomp}
\usepackage{xcolor}

\usepackage{mathtools}
\usepackage{amsmath}
\usepackage{amssymb}
\usepackage{algorithm}
\usepackage{algpseudocode}
\usepackage{accents}

\usepackage{tikz,pgfplots}
\usetikzlibrary{angles,quotes}
\pgfplotsset{compat=1.13}

\newtheorem{lemma}{Lemma}

\title{\LARGE \bf Modular  Robot and Landmark Localisation Using Relative Bearing Measurements}
\author{Behzad Zamani and Jochen Trumpf and Chris Manzie
\thanks{B. Zamani and C. Manzie are with the Department of Electrical and Electronic Engineering, 
University of Melbourne,
Melbourne, Australia
{\tt\small behzad.zamani@unimelb.edu.au, manziec@unimelb.edu.au}}
\thanks{J. Trumpf is with the School of Engineering, The Australian National University, Canberra, Australia
{\tt\small Jochen.Trumpf@anu.edu.au}}}

\begin{document}

\maketitle
\thispagestyle{empty}
\pagestyle{empty}

\begin{abstract}
    In this paper we propose a modular nonlinear least squares filtering approach for systems composed of independent subsystems. The state and error covariance estimate of each subsystem is updated independently, even when a relative measurement simultaneously depends on the states of multiple subsystems. We integrate the Covariance Intersection (CI) algorithm as part of our solution in order to prevent double counting of information when subsystems share estimates with each other. An alternative derivation of the CI algorithm based on least squares estimation makes this integration possible.  We particularise the proposed approach to the robot-landmark localization problem. In this problem, noisy measurements of the bearing angle to a stationary landmark position measured relative to the SE(2) pose of a moving robot couple the estimation problems for the robot pose and the landmark position.  In a randomized simulation study, we benchmark the proposed modular method against a monolithic joint state filter to elucidate their respective trade-offs. In this study we also include variants of the proposed method that achieve a graceful degradation of performance with reduced communication and bandwidth requirements.
\end{abstract}

\section{Introduction}
Many robotic applications are comprised of multiple separate subsystems, each receiving and processing measurements that inform the decision making of the overall system.  
 Simultaneous Localization and Mapping (SLAM)~\cite{durrant2006simultaneous, ebadi2023present}, target tracking~\cite{bar2004estimation}, collaborative localization~\cite{roumeliotis2002distributed}, self-driving cars~\cite{patole2017automotive}, and augmented/virtual reality headsets~\cite{park2021review} are examples that all  share this characteristic. In an ideal scenario, the subsystem states would all be concatenated into a single joint state, and a state estimator designed for the overall system.  However in practice, the development of complex systems undergoes mutliple design iterations, and there is motivation to develop modular approaches to the measurement and fusion of signals.  Modularity helps with managing the complexity of the design (such as the size of state vector representations and error covariance matrices), separation of concerns, debugging, interoperability and change management. 

In this paper we propose a modular nonlinear least squares filtering approach for systems composed of independent subsystems. We update the state and error covariance estimate of each subsystem independently, even when a relative measurement simultaneously depends on the states of multiple subsystems. In this setup, information exchanged between the subsystems is not statistically independent and does not contain (tracked) cross correlations~\cite{uhlmann2003covariance,ChenMehra2002estimation}. 
Kalman filter~\cite{Kalman} based fusion algorithms or similar methods are not applicable as they either require that statistically identically and independently distributed signals are exchanged or they require that the cross covariance information is available or tracked. We integrate the Covariance Intersection (CI) algorithm~\cite{julierUhlmann97CI} as part of our solution in order to prevent double counting of information when the information flow graph describing how subsystems share estimates contains a loop. We provide an alternative derivation of the CI algorithm based on least squares filtering to make this integration possible. Three variants of the proposed modular approach will be discussed that offer reduced requirements on communication, computation or both while achieving a graceful degradation of performance compared to our main results.

We particularise the proposed approach to the problem of robot-landmark localization using relative bearing measurements. In this problem, noisy measurements of the bearing angle to a stationary landmark position measured relative to the SE(2) pose of a moving robot couple the estimation problems for the robot pose and landmark position. We derive nonlinear state and error covariance estimation algorithms that tackle this problem in a modular and numerically robust manner.  In a randomized simulation study, we benchmark the proposed methods against each other as well as against a non-modular joint state filter to elucidate their respective trade-offs. 

Similar ideas to our modular estimation approach can be found in~\cite{julier2007using, arambel2001covariance, roumeliotisCI,chang2021resilient} and the related papers they cite. In~\cite{julier2007using} a CI based distributed robot landmark update method is proposed towards an EKF based formulation of the SLAM problem. Cooperative self localisation algorithms based on CI are explored in~\cite{arambel2001covariance, chang2021resilient} where every subsystem estimates the joint state of the overall network. The approach in~\cite{roumeliotisCI} is closest to ours, albeit specialised to an EKF-CI based multi-robot cooperative localisation algorithm. In contrast to these approaches, in this paper we offer a fully modular nonlinear least squares filtering approach to the highly nonlinear problem of bearing based robot-landmark localisation.          

The remainder of the paper is organised as follows. Section~\ref{sec:prob} introduces the modular state estimation problem. A nonlinear least squares approach to address this problem is proposed in Section~\ref{sec:theory}. Section~\ref{sec:robot-landmark} particularises the proposed modular least squares results to the problem of robot landmark localisation using relative bearing measurements and robot specific measurements. A randomised simulation study is provided in Section~\ref{sec:sims} that benchmarks the proposed modular methods against a non-modular classical approach. Section~\ref{sec:conclusions} concludes the paper.  

%

\section{Problem Formulation}\label{sec:prob}
Consider a plant with two states 
$x_1\in\mathbb{R}^{n_1}$ and $x_2\in\mathbb{R}^{n_2}$. Let us assume there exist two independent prior estimates $(\hat{x}_{i}(0),P_{i}(0))$, $i=1,2$, where
\begin{align}
    &\hat{x}_{i}(0)=\mathcal{E}[x_i(0)], \notag \\
    &P_{i}(0)=\mathcal{E}[\big(x_i(0)-\hat{x}_{i}(0)\big)\big(x_i(0)-\hat{x}_{i}
    (0)\big)^{\top}],\\
    &P_{1,2}(0)=\mathcal{E}[\big(x_1(0)-\hat{x}_{1}(0)\big)\big(x_2(0)-\hat{x}_{2}(0)\big)^{\top}]=0_{n_1\times n_2} \notag
\end{align}
and $\mathcal{E}[\cdot]$ denotes the expected value.
Assume further that a set of measurements $(z_i,\;W_i)$ and $(z_{1,2}, W_{1,2})$ are regularly available that provide information on $x_i$ or on both $x_1$ and $x_2$, respectively. Here, we denote by $z_i\in\mathbb{R}^{p_i}$ a measurement signal on $x_i$ with dimension $p_i$ and by $W_i\in\mathbb{R}^{p_i\times p_i}>0$ a positive definite estimated measurement error covariance of $z_i$. Similarly, $z_{1,2}\in\mathbb{R}^{p_{1,2}}$ is the measurement signal on both states $x_1$ and $x_2$ with dimension $p_{1,2}$ and estimated measurement error covariance $W_{1,2}\in\mathbb{R}^{p_{1,2}\times p_{1,2}}>0.$ 

The standard approach in signal processing is to use these measurements to update the joint state and error covariance estimate $(\hat{x}\in\mathbb{R}^{n_1+n_2}$, $P\in\mathbb{R}^{(n_1+n_2)\times(n_1+n_2)})$. The latter will involve the individual error covariances $P_1$ and $P_2$ as well as the cross covariance $P_{1,2}$. The cross covariance does not remain zero if the chosen estimation algorithm stochastically couples the estimates $\hat{x}_1$ and $\hat{x}_2$. In this paper, we aim to recursively update the subsystem estimates $(\hat{x}_i,P_i)$ in a modular fashion. Figure~\ref{fig:block} illustrates the modular block diagram of this approach. Each subsystem receives any measurement that bears information on its own state. In order to process relative measurements, which depend on both subsystem states, subsystems have to share their state and error covariance estimates with each other. However, cross covariances $P_{1,2}$ will not be tracked so that the subsystems can be decoupled and designed independently. The problem is how to perform modular decoupled state estimation without causing problems due to information double counting as a consequence of not tracking the cross correlation of information (the cross covariance $P_{1,2}$).  

Note that necessarily some optimality and/or accuracy will be lost in the modular approach as compared to a joint state estimator. 
Nevertheless, the modularity of the approach may make it an attractive option especially for large scale complex plants.

\begin{figure}
    \centering
\begin{tikzpicture}
      \node[rectangle, align=center, draw, minimum width=2em, minimum height=2em] (rect1) at (2,0){Subsystem 1\\state: $x_1$};
      
  \node[rectangle, align=center, draw, minimum width=2em, minimum height=2em] (rect2) at (6,0) {Subsystem 2 \\ state: $x_2$};
  
  \node[circle, align=center, draw, minimum size=.2em] (circ3) at (4,3) {Relative \\ Measurement};
  
\node[circle, draw, align=center, minimum size=.2em] (circ1) at (1.5,2.8) {$x_1$ \\Measurement};

\node[circle, align=center, draw, minimum size=.2em] (circ2) at (6.5,2.8) {$x_2$ \\Measurement};

\node[container, fit=(rect1) (rect2)] (plant) {};
\node at (plant.north west) [below right, node distance=0 and 0] {Plant};

  \draw[->] (rect1.350) -- (rect2.190) node[midway, below] {$(\hat{x}_1,P_1)$};
  \draw[->] (rect2.170) -- (rect1.10) node[midway, above] {$(\hat{x}_2,P_2)$};
  \draw[->] (circ3) -- (plant) node[midway, left] {$(z_{1,2},W_{1,2})$};
  \draw[->] (plant.90) -- (rect2.90) {};
  \draw[->] (plant.90) -- (rect1.90) {};
  \draw[->] (circ1) -- (rect1);
  \node at (circ1.south) [below left, shift=({3mm,1mm})] {$(z_{1},W_{1})$};
  \draw[->] (circ2) -- (rect2);
  \node at (circ2.south) [below right, shift=({-1.65cm,1mm})] {$(z_{2},W_{2})$};
\end{tikzpicture}    
    \caption{Modular Architecture.}
    \label{fig:block}
\end{figure}
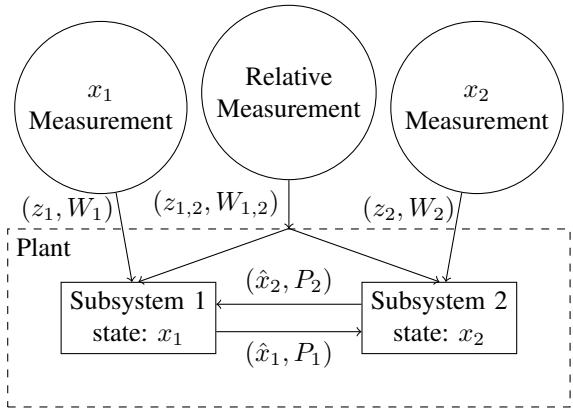

\section{Modular processing of nonlinear relative measurements}\label{sec:theory}
In this section we propose a method for updating the estimate of a sub-module state ($x_1$) based on a relative measurement ($z_{1,2}$) that also depends on the state ($x_2$) of another sub-module. Joint plant state estimation or subsystem state estimation based on measurements only depending on their own states can be done in the usual way, e.g. using the Extended Kalman Filter (EKF)~\cite{Anderson} in the case of nonlinear problems, but this approach does not work for our problem as it would require to track the cross covariance. 

Let $(\hat{x}_1,P_1)$ and $(\hat{x}_2,P_2)$  denote the estimates for $x_1$ and $x_2$ and their error covariances, respectively. The following quadratic costs measure the uncertainty of these estimates,
\begin{equation}
    \begin{split}
        & \frac{1}{2}\lVert x_1-\hat{x}_1\rVert^2_{P_1^{-1}},\; \frac{1}{2}\lVert x_2-\hat{x}_2\rVert^2_{P_2^{-1}},
    \end{split}
\end{equation}
where $\lVert x\rVert^2_{P^{-1}}\triangleq x^\top P^{-1}x$.
Let $(z_{1,2},W_{1,2})$ denote a measurement that nonlinearly depends on both $x_1$ and $x_2$ and its estimated measurement error covariance.
Denote by $h_{1,2}:\mathbb{R}^{n_{1}}\times \mathbb{R}^{n_{2}}\times\mathbb{R}^{p_{1,2}}\longrightarrow\mathbb{R}^{p_{1,2}}$ a nonlinear \emph{error map}, i.e. a map that takes the value zero when the measurement is exact. Then the following quadratic cost is centered around zero and measures the uncertainty of $z_{1,2}$ with respect to $x_1$ and $x_2$,
\begin{equation}\label{eq:mutualcost}
     \frac{1}{2} \lVert h_{1,2}(x_1,x_2, z_{1,2}) \rVert^2_{W_{1,2}^{-1}}.
\end{equation}
This general setup specialises to useful and familiar cases such as linear relative measurements $ h^l_{1,2}(x_1,x_2, z_{1,2}) = z_{1,2} - (x_1 - x_2)$, or the nonlinear relative measurements considered in Section~\ref{sec:robot-landmark}. 
The task of updating the estimate for $x_1$ consists of fusing the prior estimates and the relative measurement information by minimising a cost functional. 
In a first step, we will minimise a combined measure of uncertainty of the prior estimate $\hat{x}_2$ of $x_2$ and the relative measurement $z_{1,2}$ with respect to $x_2$. In a second step we will then fuse in the prior estimate $\hat{x}_1$ of $x_1$. 
Consider the least squares cost capturing the overall uncertainty of information about $x_2$,
\begin{equation}\label{eq:LSC}\begin{split}
        J_{1,2}(x_1,x_2) &= \frac{1}{2} \lVert h_{1,2}(x_1,x_2,z_{1,2}) \rVert^2_{W_{1,2}^{-1}} \\
        &\hspace{2cm} + \frac{1}{2} \lVert x_2 -\hat{x}_2\rVert^2_{P_2^{-1}}.
\end{split}\end{equation}
\begin{lemma}[Nonlinear Parallel Sum]\label{lem:NPS}
Let $H_i$ $\triangleq$ $\nabla_{x_i}h_{1,2}(\hat{x}_1, \hat{x}_2,z_{1,2})$, $i=1,2$.
 Given the prior estimate $(\hat{x}_2,P_2)$ and the measurement $(z_{1,2},W_{1,2})$, the optimal value of the overall uncertainty cost~\eqref{eq:LSC} in terms of $x_1$, up to second order approximation error, is given by
    \begin{equation}
    \begin{split}
        &J^*_{1,2}(x_1) = \frac{1}{2} \lVert h_{1,2}(x_1,\hat{x}_2,z_{1,2}) \rVert^2_{\tilde{W}_{1}^{-1}}, 
        \end{split}
    \end{equation}
    where $\tilde{W}_{1}\triangleq W_{1,2} + H_{2}P_2H_{2}^{\top}$.
\end{lemma}
\begin{proof}
    Let $x_2^*\triangleq \underaccent{x_2 }{\mbox{argmin}}\;J_{1,2}(x_1,x_2)$. 
    The first order minimality condition with respect to $x_2$ is
\begin{equation*}\begin{split}
    &(\nabla_{x_2}h_{1,2}(x_1,x^*_2,z_{1,2}))^{\top}W_{1,2}^{-1} h_{1,2}(x_1,x^*_2,z_{1,2})\\
    &\hspace{2cm}+P_2^{-1}(x^*_2 -\hat{x}_2)=0.
    \end{split}
    \end{equation*}
    By Taylor expansion with respect to $x_2$ around $\hat{x}_2$ and evaluation at $x_2= x_2^*$ we get $h_{1,2}(x_1,x^*_2,z_{1,2})= h_{1,2}(x_1,\hat{x}_2,z_{1,2}) + H_2(x^*_2-\hat{x}_2)$. Then, since $H_2^{\top}W_{1,2}^{-1}H_2 + P_2^{-1}$ is invertible,
    \begin{equation*}
        x^*_2 = \hat{x}_2 - \big(H_2^{\top}W_{1,2}^{-1}H_2 + P_2^{-1}\big)^{-1}\!H_2^{\top}W_{1,2}^{-1}  h_{1,2}(x_1,\hat{x}_2,z_{1,2}).
    \end{equation*}
    Replacing $x_2^*$ into~\eqref{eq:LSC} and applying the matrix inversion lemma completes the proof.
\end{proof}
Note that the information about $x_2$ contained in the relative measurement $z_{1,2}$ is independent of the prior estimate $\hat{x}_2$ of $x_2$. This is because that information is only dependent on the true signals $x_1$ and $x_2$ and the measurement errors inherent to the particular measurement device that produced $z_{1,2}$. This explains why we are able to perform a fusion step in a standard least squares filtering fashion (Lemma~\ref{lem:NPS}).

This is no longer the case in the second step where the objective is to fuse the prior information on $x_1$, $(\hat{x}_1,P_1)$ with the information encoded in the functional $J^*_{1,2}(x_1)$. This is because $J^*_{1,2}(x_1)$ depends on $(\hat{x}_2,P_2)$ which may have been produced based on $(\hat{x}_1,P_1)$ itself in a prior computation round. A fusion step in standard least squares filtering fashion then may lead to double counting of prior information. 
This phenomenon is well known and has given rise to extensive research on safe fusion of correlated information in the absence of cross covariance information, see for example \cite{julierUhlmann97CI, Chenfusion02}. The Covariance Intersection (CI) algorithm~\cite{julierUhlmann97CI} provides the optimal safe solution to this problem \cite{reinhardt2015}. 
Let $0\leq\alpha \leq 1$ and consider two possibly correlated estimates for $x$, $(\hat{x},P)$ and $(\hat{x}',P')$. Then the CI fusion algorithm yields
    \begin{equation}\label{eq:CI}
        \begin{split}
            x^+ &= P^+\bigg(\alpha^*P^{-1}\hat{x} + (1-\alpha^*){P'}^{-1}\hat{x}'\bigg) \\
            &= \hat{x}-(1-\alpha^*)P^+ {P'}^{-1}(\hat{x}-\hat{x}'),\\
            P^+ &= \bigg(\alpha^* P^{-1}+ (1-\alpha^*) {P'}^{-1} \bigg)^{-1},\\
            \alpha^* &= \underaccent{0\leq\alpha \leq 1}{\mbox{argmin}}\det\bigg(\alpha P^{-1}+ (1-\alpha) {P'}^{-1} \bigg)^{-1}.
        \end{split}
    \end{equation}
Note that the second formula for $x^+$ given above is not standard and makes explicit how the prior estimate $\hat{x}$ is updated based on a new, possibly correlated estimate $\hat{x}'$.
\begin{lemma}[Least Squares CI]\label{lemma:LSCI}
    
    The first part of the CI algorithm~\eqref{eq:CI} is the minimiser of the convex combination least squares cost
    \begin{equation}
        J(x)=\frac{\alpha^*}{2}\lVert x-\hat{x}\rVert^2_{P^{-1}} + \frac{(1-\alpha^*)}{2}\lVert x-\hat{x}'\rVert^2_{{P'}^{-1}}.
    \end{equation}
\end{lemma}
\vspace{2mm}
\begin{proof}
    Follows by direct calculation.
\end{proof}
 
Now we apply Lemma~\ref{lemma:LSCI} to the fusion of information about $x_1$. 

\subsubsection*{Modular Fusion}
    Let $(\hat{x}_1,P_1)$ and $(\hat{x}_2,P_2)$ denote two possibly correlated estimates for $x_1$ and $x_2$, respectively. Let the cost~\eqref{eq:mutualcost} capture the independent uncertainty of a relative measurement $(z_{1,2},W_{1,2})$ which bears information on both $x_1$ and $x_2$.
    Recall $H_i = \nabla_{x_i}h_{1,2}(\hat{x}_1, \hat{x}_2,z_{1,2})$, $i=1,2$, $\tilde{W}_{1} = W_{1,2} + H_{2}P_2H_{2}^{\top}$ and denote $\tilde{W}_{2}\triangleq W_{1,2} + H_{1}P_1H_{1}^{\top}$.
     Then, the following algorithm 
    \begin{equation}\label{eq:modular_fusion}
        \begin{split}
            &\hat{x}_1^+ = \hat{x}_1  -(1-\alpha_1^*)P_1^+H_1^{\top}{\tilde{W}}^{-1}_{1}h_{1,2}(\hat{x}_1,\hat{x}_2,z_{1,2}),\\
            &P_1^+ = \bigg(\alpha_1^* P_1^{-1}+ (1-\alpha_1^*) H_1^{\top}{\tilde{W}}^{-1}_{1}H_1 \bigg)^{-1},\\
            &\alpha_1^* = \underaccent{0\leq\alpha \leq }{\mbox{argmin}}\det\bigg(\alpha P_1^{-1}+ (1-\alpha) H_1^{\top}{\tilde{W}}^{-1}_{1}H_1 \bigg)^{-1},\\
            &\hat{x}_2^+ = \hat{x}_2  -(1-\alpha_2^*)P_2^+H_2^{\top}{\tilde{W}}^{-1}_{2}h_{1,2}(\hat{x}_1,\hat{x}_2,z_{1,2}),\\
            &P_2^+ = \bigg(\alpha_2^* P_2^{-1}+ (1-\alpha_2^*) H_2^{\top}{\tilde{W}}^{-1}_{2}H_2 \bigg)^{-1},\\
            &\alpha_2^* = \underaccent{0\leq\alpha \leq }{\mbox{argmin}}\det\bigg(\alpha P_2^{-1}+ (1-\alpha) H_2^{\top}{\tilde{W}}^{-1}_{2}H_2 \bigg)^{-1}
        \end{split}
    \end{equation}
    yields a modular fusion estimate for the subsystems in figure~\ref{fig:block}. Note that this algorithm inherits the consistency properties of CI~\cite{julierUhlmann97CI}.

\section{Robot-Landmark Localisation Using Bearing Measurements}\label{sec:robot-landmark}
Now let us consider the Robot-Landmark Localisation problem and apply the proposed modular fusion algorithm~\eqref{eq:modular_fusion}.
Table~\ref{tab:mod} contains the mapping between the relevant parameters in Section~\ref{sec:theory} and the ones that we will progressively introduce here.
\begin{table}[!h]
    \centering
    \begin{tabular}{cc}
        Section~\ref{sec:theory} & Section~\ref{sec:robot-landmark}\\ $(\hat{x}_1,P_1)$ & $(\hat{X}_r,P_r)$\\  $(\hat{x}_2,P_2)$ & $(\hat{p}_l,P_l)$\\  $(z_{1,2},W_{1,2})$ & $(z\text{ in Eq.~\eqref{eq:bearing_meas}}, \sigma^2_{\theta_{rl}})$ \\
          $h(x_1,x_2,z_{1,2})$  & $(I-zz^{\top})R_r^{\top}(p_l-p_r)$ \\
    \end{tabular}
    \caption{Parameter mapping between the development in Section~\ref{sec:theory} and the bearing based robot-landmark localization problem}
    \label{tab:mod}
\end{table}
Consider a stationary landmark and a mobile robot equipped with a bearing measurement sensor. The robot's task is to accurately localise the landmark in the global reference frame.  For this task, the mobile robot aims to produce an estimate of its own pose $X_r\in SE(2)$, comprised of its attitude $R_r\in SO(2)$ and its position $p_r\in\mathbb{R}^2$ with respect to the global reference frame, as well as an estimate of the landmark position $p_l\in\mathbb{R}^2$. The robot's pose estimate $(\hat{X}_r,P_r)$ is obtained via dead-reckoning assuming that measurements of linear and angular velocity and robot state measurements from GPS and compass are available. Relative bearing measurements $z\in S^1$ to the landmark position are obtained onboard of the robot and used to update the estimates $(\hat{X}_r,P_r)$ and $(\hat{p}_l,P_l)$ for both the robot pose and the landmark position.  

\subsection{Mobile Robot Model}\label{sec:robot}
Let $x_r, y_r\in\mathbb{R}$ denote the x-y coordinates and let $\theta_r\in S^1$ denote the heading angle of the robot (measured counter-clockwise with respect to the $x$-axis) with respect to the global reference frame. Denote 
\begin{equation}
    p_r\triangleq\begin{bmatrix}
        x_r\\ y_r
    \end{bmatrix},
    \; R_r\triangleq \begin{bmatrix}
        \cos(\theta_r) & -\sin(\theta_r)\\\sin(\theta_r) & \cos(\theta_r)
    \end{bmatrix}.
\end{equation}
We consider a non-holonomic movement model as a subset of SE(2) behaviors, namely the unicycle model where the body-fixed frame velocity of the robot is constrained to the forward direction. Denote by $v_r\in\mathbb{R}$ the forward speed of the robot and by $\omega_r\in\mathbb{R}$ the body-fixed frame yaw rate, respectively.
Denote by $X_r(k)\triangleq [
        x_r(k)\;
        y_r(k)\;
        \theta_r(k)]^{\top}$ and $u_r(k)\triangleq [v_r(k)\;w_r(k)]^{\top}$ the robot pose and twist at time $k$, respectively. 
Note that we are using the more compact $3$-vector representation of $SE(2)$ and not the $3\times3$-matrix representation.
Given a discretization time step $\tau$ the first order Euler discrete-time unicycle model of the mobile robot is 
 \begin{equation}\label{eq:discrete_rtate_kin}
 \begin{split}
    X_r(k+1) &= f_r(X_r(k),u_r(k))\\
    &= X_r(k) + \tau\begin{bmatrix}
        \cos(\theta_r(k)) & 0\\\sin(\theta_r(k)) & 0\\ 0 & 1
    \end{bmatrix}\begin{bmatrix}
        v(k) \\ w(k)
    \end{bmatrix}.
    \end{split}
    \end{equation}
\subsection{Robot Pose Filter Using Robot Specific Measurements}\label{sec:robotequations}
In this section, we assume that we have access to regular forward speed and yaw rate measurements as well as occasional  state measurements of the mobile robot.
We use the prediction and update steps of the EKF to estimate the robot pose together with its error covariance. 

\subsubsection{Prediction} Denote by $v^m = v_r +\delta_v$ and $w^m = w_r+\delta_w$ the measured forward speed and yaw rate, respectively. We assume that these are corrupted by zero mean Gaussian errors $\delta_v\sim \mathcal{N}(0,\sigma_v^2)$ and $\delta_w\sim \mathcal{N}(0,\sigma_w^2)$ with $\sigma_v, \sigma_w>0$ denoting their standard deviations, respectively. Denote by $u_m(k)\triangleq [v^m(k)\; w^m(k)]^{\top}$ the measured twist at time $k$. 

 The discrete-time prediction of the estimated pose is then the expected value of~\eqref{eq:discrete_rtate_kin},  
\begin{equation}\label{eq:discrete_kin}
        \hat{X}_r(k+1) = f_r(\hat{X}_r(k),u^m_r(k)),\; \hat{X}_r(0)=\hat{X}_{r,0}.
\end{equation}
Here, $\hat{X}_{r,0}\triangleq [
    \hat{x}_{r,0}\;\hat{y}_{r,0}\;\hat{\theta}_{r,0}]^{\top}$ denotes the initial pose estimate of the robot. Denote by $P_{r,0}>0$ the estimated error covariance of the initial pose estimate. The EKF prediction step yields
\begin{equation}\label{eq:PredictionCov}
    \begin{split}
        P_r(k+1) &= A_r(k)P_r(k)A_r^{\top}(k) + B_r(k)QB_r^{\top}(k),\\
        P_r(0) &= P_{r,0},\\ 
        A_r(k)&\triangleq \nabla_{X_r}f_r(\hat{X}_r(k),u_m(k))\\
        &=\begin{bmatrix}
            I_{2\times2} & \begin{matrix}-\tau\sin(\hat{\theta}_r(k))v^m(k)\\
            \tau\cos(\hat{\theta}_r(k))v^m(k)\end{matrix}
            \\
            0_{1\times2} & 1
        \end{bmatrix},\\
        B_r(k)&\triangleq \nabla_{u_r}f_r(\hat{X}_r(k),u_m(k))\\
        &=\begin{bmatrix}
            \tau\cos(\hat{\theta}_r(k)) & 0 \\
            \tau\sin(\hat{\theta}_r(k)) & 0\\
            0 &\tau
        \end{bmatrix},\\
        Q&\triangleq\begin{bmatrix}
            \sigma^2_v & 0 \\ 0 & \sigma_w^2
        \end{bmatrix}.
    \end{split}
\end{equation}
\subsubsection{Update} Let us assume that we obtain occasional robot state measurements from devices like GPS and magnetometer onboard the moving robot, 
\begin{equation}\label{eq:robot_measurement}
    Y_r(k) = X_r(k) + \delta_r(k).
\end{equation}
Here, $Y_r(k)\in\mathbb{R}^{3}$ is the full state noisy measurement of $X_r(k)$ and $\delta_r\sim \mathcal{N}(0, \Sigma_r)$ is the measurement error with measurement error covariance matrix $\Sigma_r\in\mathbb{R}^{3\times 3}>0$. We employ an EKF update step to update the robot state and error covariance estimates,
\begin{equation}
    \begin{split}
    K_r(k)&= P_r(k) (P_r(k) + \Sigma_r)^{-1},\\ 
    \hat{X}_r^+ &= \hat{X}_r(k) + K_r(k)(Y_r(k) - \hat{X}_r(k)), \\
    P_r^+ &= (I-K_r(k)) P_r(k). 
    \end{split}
\end{equation}


\subsection{Relative Bearing Measurement }\label{sec:bearing_meas}
   At each time instance $k$ the robot measures a bearing angle $\theta_{rl}^m(k)\in S^1$ to the landmark position,
 \begin{equation}\label{eq:bearing_measurement}
     \theta_{rl}^m(k) = \theta_{rl}(k) + \delta_{\theta_{rl}}(k), 
 \end{equation}
where $\theta_{rl}(k),\delta_{\theta_{rl}}(k)\in S^1$ denote the true bearing angle and measurement error angles, respectively. We assume that $\delta_{\theta_{rl}}(k)\sim \mathcal{N}(0,\sigma^2_{\theta_{rl}})$ where $\sigma_{\theta_{rl}}\in S^1$ is the standard deviation of the measurement error. 
Denote the true bearing and measured bearing by
\begin{equation}\label{eq:bearing_meas}
\begin{split}
    &\phi(k)\triangleq \begin{bmatrix}\cos(\theta_{rl}(k)) \\ \sin(\theta_{rl}(k))  \end{bmatrix} = \frac{R_r^{\top}(k)(p_l - p_r(k))}{\Vert p_l- p_r(k)\Vert}
    \ \text{ and}\\
    & z(k)\triangleq \begin{bmatrix}\cos(\theta_{rl}^m(k)) \\ \sin(\theta_{rl}^m(k))  \end{bmatrix},
\end{split}
\end{equation}
respectively.    
We can use this relative measurement to update the state and error covariance estimates of the robot as well as the state estimate $\hat{p}_l(k)$ and error covariance estimate $P_l(k)\in\mathbb{R}^{2\times 2}$ of the landmark. 
The following cost captures the uncertainty of the relative bearing measurement $z(k)$ with respect to both robot and landmark states, 
 

\begin{equation}\label{eq:jointcost}\begin{split}
    &J_z(p_l,R_r,p_r) = \\
    &\hspace{0.7cm}\frac{1}{2\sigma_{\theta_{rl}}^2}\Vert (I- z(k)z(k)^{\top})R_r^{\top}(k)(p_l-p_r(k))\Vert ^2. 
\end{split}\end{equation}
This cost is as proposed in the target tracking literature~\cite{ lingren1978position} and captures the measurement error in the only non-trivial subspace, perpendicular to the measured bearing $z(k)$ (in the body-fixed frame of the robot at time $k$). Note that the expected value of the measured bearing direction $z$ is equal to the true bearing direction $\phi$. Moreover, in expectation, $I-zz^{\top}$ is a projection onto the subspace (direction) perpendicular to the true bearing direction $\phi$. Hence, $(I-zz^{\top})R_r^{\top}(p_l-p_r) = (I-zz^{\top})\phi\Vert p_l- p_r(k)\Vert=0$ in expectation. Thus, minimising this cost picks a robot-landmark estimate with the least expected bearing error with respect to the measured bearing direction. 



\subsection{Modular Robot-Landmark Localization}\label{sec:mod}
In this section we apply the modular fusion algorithm~\eqref{eq:modular_fusion} to the robot-landmark localization problem.  
Denote
\begin{equation}\label{eq:rl_defs}\begin{split}
 z^{\perp} &= \begin{bmatrix}
                -\sin(\theta_{rl}^m) \\ \cos(\theta_{rl}^m)
              \end{bmatrix}, 
 \ \hat{R}_r \triangleq \begin{bmatrix}
                         \cos(\hat{\theta}_r) & -\sin(\hat{\theta}_r) \\ \sin(\hat{\theta}_r) & \cos(\hat{\theta}_r)
                        \end{bmatrix} \\
 \tilde{z} &\triangleq \hat{R}_rz^{\perp}, \\
 U_r &\triangleq \begin{bmatrix}
            -I_{2\times 2}
             & \begin{matrix} (\hat{y}_l-\hat{y}_r)\\
            -(\hat{x}_l-\hat{x}_r)\end{matrix}
        \end{bmatrix}, \\
 H_r &\triangleq (I-zz^{\top})\hat{R}^{\top}_rU_r = \hat{R}^{\top}_r\tilde{z}\tilde{z}^{\top}U_r, \\
 \tilde{W}_r &\triangleq \left(\sigma_{\theta_{rl}}^2I + H_rP_rH^{\top}_r\right),\ \gamma_r \triangleq \sqrt{\tilde{ z}^{\top}U_rP_rU_r^{\top}\tilde{z}}, \\
 H_l &\triangleq (I-zz^{\top})\hat{R}_r^{\top}, \\
 \tilde{W}_l &\triangleq \left(\sigma_{\theta_{rl}}^2I + H_lP_lH^{\top}_l\right),\ \gamma_l \triangleq \sqrt{\tilde{ z}^{\top}P_l\tilde{z}}.
\end{split}\end{equation}

\subsubsection{Landmark Subsystem} 
Applying Lemma~\ref{lem:NPS} to the joint cost function~\eqref{eq:jointcost} yields
    \begin{equation}
        J^\ast_z(p_l) = \frac{1}{2}\lVert (I- z(k)z(k)^{\top})\hat{R}_r^{\top}(k)(p_l-\hat{p}_r(k))\rVert ^2_{\tilde{W}_r^{-1}}.
    \end{equation}
From~\eqref{eq:CI}, the landmark fusion algorithm is
\begin{equation}\label{eq:FKT_fusion}
    \begin{split}
        \hat{p}^{+}_l &= \hat{p}_l - (1-\alpha_l^*)P^{+}_l  S_l^{-1}(\hat{p}_l- \hat{p}_r), \\
        P^{+}_l &= \left( \alpha_l^*P_l^{-1} + (1-\alpha_l^*)S_l^{-1} \right)^{-1},\\
        S_l^{-1} &\triangleq \tilde{z}\tilde{ z}^{\top}\hat{R}_r\tilde{W}_r^{-1}\hat{R}_r^{\top}\tilde{z}\tilde{ z}^{\top}\\
        &=
        \tilde{z}\tilde{ z}^{\top}(\sigma_{\theta_{rl}}^2I +\tilde{z}\tilde{ z}^{\top}U_rP_rU_r^{\top}\tilde{z}\tilde{ z}^{\top})^{-1}\tilde{z}\tilde{ z}^{\top} \\&= \frac{1}{\sigma_{\theta_{rl}}^2+\gamma_r^2}\tilde{z}\tilde{ z}^{\top}.
    \end{split}
\end{equation}
Using the Sherman-Morrison formula to eliminate matrix inverses this simplifies to
\begin{equation}\label{eq:Baseline_fusion2}
\begin{split}
        \hat{p}^{+}_l  &= \hat{p}_l - (1-\alpha_l^*)\frac{P^{+}_l} {\sigma_{\theta_{rl}}^2+\gamma_r^2}\tilde{z}\tilde{ z}^{\top}(\hat{p}_l- \hat{p}_r), \\
        P^{+}_l &= \frac{1}{\alpha_l^*}\left(P_l -\frac{ P_l\tilde{z}\tilde{z}^{\top}P_l}{\frac{\alpha_l^*(\sigma_{\theta_{rl}}^2+\gamma_r^2)}{1-\alpha_l^*}+\tilde{z}^{\top}P_l\tilde{z} }\right),\\
        \alpha_l^* &= \underaccent{0\leq\alpha \leq 1}{\mbox{argmin}}\det\frac{1}{\alpha}\left(P_l -\frac{ P_l\tilde{z}\tilde{z}^{\top}P_l}{\frac{\alpha(\sigma_{\theta_{rl}}^2+\gamma_r^2)}{1-\alpha}+\tilde{z}^{\top}P_l\tilde{z} }\right).
    \end{split}
\end{equation}  

\subsubsection{Robot Subsystem}
Applying Lemma~\ref{lem:NPS} to the joint cost function~\eqref{eq:jointcost} yields
    \begin{equation}
        J^\ast_z(X_r) = \frac{1}{2}\lVert (I- z(k)z(k)^{\top})R_r^{\top}(\hat{p}_l(k)-p_r)\rVert ^2_{\tilde{W}_l^{-1}}.
    \end{equation}
From~\eqref{eq:CI}, the robot fusion algorithm is
\begin{equation}\label{eq:FKV_fusion}
    \begin{split}
        \hat{X}^{+}_r  &= \hat{X}_r - (1-\alpha_r^*)P_r^{+}  U_r^{\top}S_r^{-1}(\hat{p}_l- \hat{p}_r), \\
        P_r^{+} &= \left( \alpha_r^*P_r^{-1} + (1-\alpha_r^*)U_r^{\top}S_r^{-1}U_r \right)^{-1},\\
        S_r^{-1} &\triangleq \tilde{z}\tilde{ z}^{\top}\hat{R}_r\tilde{W}_l^{-1}\hat{R}_r^{\top}\tilde{z}\tilde{ z}^{\top} \\
        &= \tilde{z}\tilde{ z}^{\top}(\sigma_{\theta_{rl}}^2I +\tilde{z}\tilde{ z}^{\top}P_l\tilde{z}\tilde{ z}^{\top})^{-1}\tilde{z}\tilde{ z}^{\top}\\
        &= \frac{1}{\sigma_{\theta_{rl}}^2+\gamma_l^2}\tilde{z}\tilde{ z}^{\top}.
    \end{split}
\end{equation}
Using the Sherman-Morrison formula to eliminate matrix inverses this simplifies to
\begin{align}\label{eq:FKV_fusion2}
        \hat{X}^{+}_r  &= \hat{X}_r - (1-\alpha_r^*)\frac{P_r^{+}} {\sigma_{\theta_{rl}}^2+\gamma_l^2}U_r^{\top}\tilde{z}\tilde{ z}^{\top}(\hat{p}_l- \hat{p}_r), \notag\\
        P_r^{+} &=\frac{1}{\alpha_r^*}\left(P_r -\frac{ P_rU_r^{\top}\tilde{z}\tilde{z}^{\top}U_rP_r}{\frac{\alpha_r^*(\sigma_{\theta_{rl}}^2+\gamma_l^2)}{1-\alpha_r^*}+\tilde{z}^{\top}U_rP_rU_r^{\top}\tilde{z} }\right),\\
        \alpha_r^* &= \underaccent{0\leq\alpha \leq 1}{\mbox{argmin}}\det \frac{1}{\alpha}\left(P_r -\frac{ P_rU_r^{\top}\tilde{z}\tilde{z}^{\top}U_rP_r}{\frac{\alpha(\sigma_{\theta_{rl}}^2+\gamma_l^2)}{1-\alpha}+\tilde{z}^{\top}U_rP_rU_r^{\top}\tilde{z} }\right). \notag
\end{align}

\subsection{Joint Robot Landmark Algorithm}\label{sec:joint}
In this section we provide a non-modular robot landmark estimation algorithm using the same measurements we used for the modular algorithm. This algorithm will serve as a benchmark for the analysis in Section~\ref{sec:sims}. 
\subsubsection{Joint Model}
 Let us form the joint robot-landmark state $X\triangleq[x_r\;y_r\;\theta_r\;x_l\;y_l]$. Then the kinematics of the joint state model in discrete time is
\begin{equation}
\begin{split}
    X(k+1) &= f(X(k),u(k)),\\
    &=X(k) + 
        \tau\begin{bmatrix}
       \begin{matrix}
           \begin{matrix}
\cos(\theta_r(k))\\\sin(\theta_r(k)) \end{matrix} & 0_{2\times1}\\0 & 1
       \end{matrix} \\0_{2\times 2}
\end{bmatrix}\begin{bmatrix}
        v(k) \\ w(k)
    \end{bmatrix}. 
\end{split}
     \end{equation}
\subsubsection{Joint State Prediction}
Given the joint model, the joint state and error covariance estimates predicted by the EKF are similar to the ones for the robot's pose only,
\begin{align}
        \hat{X}(k+1) &= f(\hat{X}(k), u_m(k)),\notag\\
        P(k+1) &= A(k)P_r(k)A^{\top}(k) + B(k)QB^{\top}(k),\notag\\
        P(0)&=\mbox{diag}(P_{r,0}, P_{l,0}),\\
        A(k)&\triangleq \nabla_{X}f(\hat{X}(k),u_m(k))=\begin{bmatrix}
            \begin{matrix}A_r(k) & 0_{3\times 2}\\
            0_{2\times 3} & I_{2\times 2}\end{matrix}
        \end{bmatrix},\notag\\
        B(k)&\triangleq \nabla_{u}f(\hat{X}(k),u_m(k))=\begin{bmatrix}\begin{matrix}
            B_r(k)\\0_{2\times2}
        \end{matrix}
            \end{bmatrix}.\notag
\end{align}
\subsubsection{GPS/Compass Robot Update}
In terms of the joint state, the measurement equation~\eqref{eq:robot_measurement} is 
\begin{equation}
    Y_r(k) = CX(k) + \delta_r(k),\;C\triangleq \begin{bmatrix}
       \begin{matrix}
           I_{3\times 3}&0_{2\times 2}
       \end{matrix}
    \end{bmatrix}.
\end{equation}
The EKF update of the joint state and error covariance estimates is
\begin{equation}
    \begin{split}
    K(k) &= P(k) C^{\top}(CP(k)C^{\top} + \Sigma_r)^{-1},\\
    \hat{X}^+ &= \hat{X}(k) + K(k)(Y_r(k) - C\hat{X}(k)), \\  
    P^+ &= (I-K(k)C) P(k). 
    \end{split}
\end{equation}

\subsubsection{Joint State Bearing Update}
The bearing measurement~\eqref{eq:bearing_measurement} contains information on both the landmark and the robot states. 
Consider the least squares cost
\begin{equation}\begin{split}
    &J(X)= \frac{1}{2}\Vert X(k) - \hat{X}(k)\Vert^2_{P^{-1}(k)} + J_{z}(X),\\
    &J_{z}(X) 
    =\frac{1}{2\sigma^2_{\theta_{rl}}} \Vert (I- zz^{\top})R_r^{\top}(p_l-p_r)\Vert^2.
\end{split} 
\end{equation}
Recall the definition of $\tilde{z}$ from equation~\eqref{eq:rl_defs} and denote
\begin{equation}\label{eq:U}
 U \triangleq \begin{bmatrix}
            -I & \begin{matrix}
                -(\hat{y}_l-\hat{y}_r)\\
            (\hat{x}_l-\hat{x}_r)
            \end{matrix} & I
        \end{bmatrix}.
\end{equation}
The least squares joint state update equations are then
\begin{equation}\label{eq:Joint_rtate_fusion}
    \begin{split}
        \hat{X}^+  &=\hat{X} - P^+\bigg( U^{\top}S^{-1}(\hat{p}_l-\hat{p}_r)\bigg),\\
        P^+ &= \big(  P^{-1} +  S^{-1}\big)^{-1}= P - \frac{PU^{\top}\tilde{z}\tilde{z}^{\top}UP}{\sigma^2_{\theta_{rl}}+ \tilde{z}^{\top}UPU^{\top}\tilde{z}},\\
        S^{-1} &\triangleq \frac{1}{\sigma^2_{\theta_{rl}}}U^{\top}\tilde{z}\tilde{z}^{\top}U.
    \end{split}
\end{equation}

\section{Simulation Study}\label{sec:sims}
In this section we provide a randomised simulation study to assess the performance of the modular robot-landmark localisation algorithm proposed in Sections~\ref{sec:robotequations} and~\ref{sec:mod}  in comparison to the non-modular (joint state) approach in Section~\ref{sec:joint} which serves as a baseline for this study. We will also include in this study three other versions of the proposed modular algorithm in Section~\ref{sec:mod} that have reduced communication requirements, reduced computation requirements or both. The reduced communication variant only allows sharing of the state estimates and not of the error covariance estimates between subsystems. This means no relative measurement weight matrix adjustment (Lemma~\ref{lem:NPS}) is required. This reduced communication method will be denoted as `Safe' as opposed to the proposed method denoted as `FSafe' (Full communication, Safe in terms of not double counting correlated information). The reduced computation variants of these two methods will neglect the fact that estimates shared between subsystems are correlated and will perform a standard least squares (or Kalman Fusion as in~\cite{julierUhlmann97CI}) instead of the CI least squares fusion step (Lemma~\ref{lemma:LSCI}). This amounts to requiring no optimal convex weight $\alpha^*$ computation ($\alpha^*\longleftarrow 1$ and $(1-\alpha^*)\longleftarrow 1$). These reduced computation methods will be denoted as `Kalman' (reduced communication and computation) and 'FKalman' (Full communication but reduced computation). The baseline non-modular approach will be denoted as 'Joint'.
\subsection{Scenario}\label{sec:sen}
Consider a $30m$ wide square area around the origin with uniformly random initial robot placement $p_r(0)\sim\mathcal{U}([-13, -13]^{\top},[13,13]^{\top})$ and initial landmark placement $p_l\sim\mathcal{U}([-7.5, -7.5]^{\top},[7.5,7.5]^{\top})$.
The robot's initial heading angle is $\theta_r\sim\mathcal{U}(0,2\pi)$.
Their initial estimates are also uniformly randomly chosen as $\hat{p}_r(0),\hat{p}_l(0)\sim\mathcal{U}([-15, -15]^{\top},[15,15]^{\top})$ and 
$\hat{\theta}_r(0)\sim\mathcal{U}(0,2\pi)$ while their initial error covariance estimates are $P_r = \mbox{diag}([100,400,(\pi/18)^2])$ and $P_l = 9000I_{2\times 2}$. 
The standard deviations of the measurement errors are $\sigma_v\sim |\mathcal{N}(0,0.25)|$, $\sigma_w\sim |\mathcal{N}(0,(\pi/90)^2)|$, $\Sigma_r=\mbox{diag}(\sigma_r^2),\;\sigma_r\sim\mathcal{N}(0,\mbox{diag}([25, 25,(7\pi/180)^2)]))$ and $\sigma_{rl}\sim |\mathcal{N}(0,(7\pi/180)^2)|$.
 
All methods receive robot twist measurements at all times but GPS/Compass (robot state) measurements at every $3$ time steps and relative bearing measurements at every $6$ time steps. The robot follows a random walk keeping inside the square area for $T=100$ time steps ($k=0,\cdots,T$) with the fixed forward speed of $v_r(k)\equiv1$ and a semi-random yaw rate of $w_r(k+1)=0.4w_r(k) + 0.6\delta$ where $w_r(0)=-0.07$ and $\delta\sim\mathcal{U}(-\pi/4,\pi/4)$ which reorients towards the origin if the robot would otherwise exit the square area.  
\subsection{Results}
We simulate $20000$ repeats of the randomised scenario in Section~\ref{sec:sen} and compare the performance of the 5 methods mentioned in Section~\ref{sec:sims}. We evaluate the performance of these methods based on the landmark estimation error they achieve at the end of each repeat, namely $e_l(T)\triangleq \Vert p_l - \hat{p}_l(T)\Vert$. Figure~\ref{fig:box} shows a box plot of the distribution of this error across all our experiments.   
\begin{figure}
    \centering
    \includegraphics[width=1\linewidth]{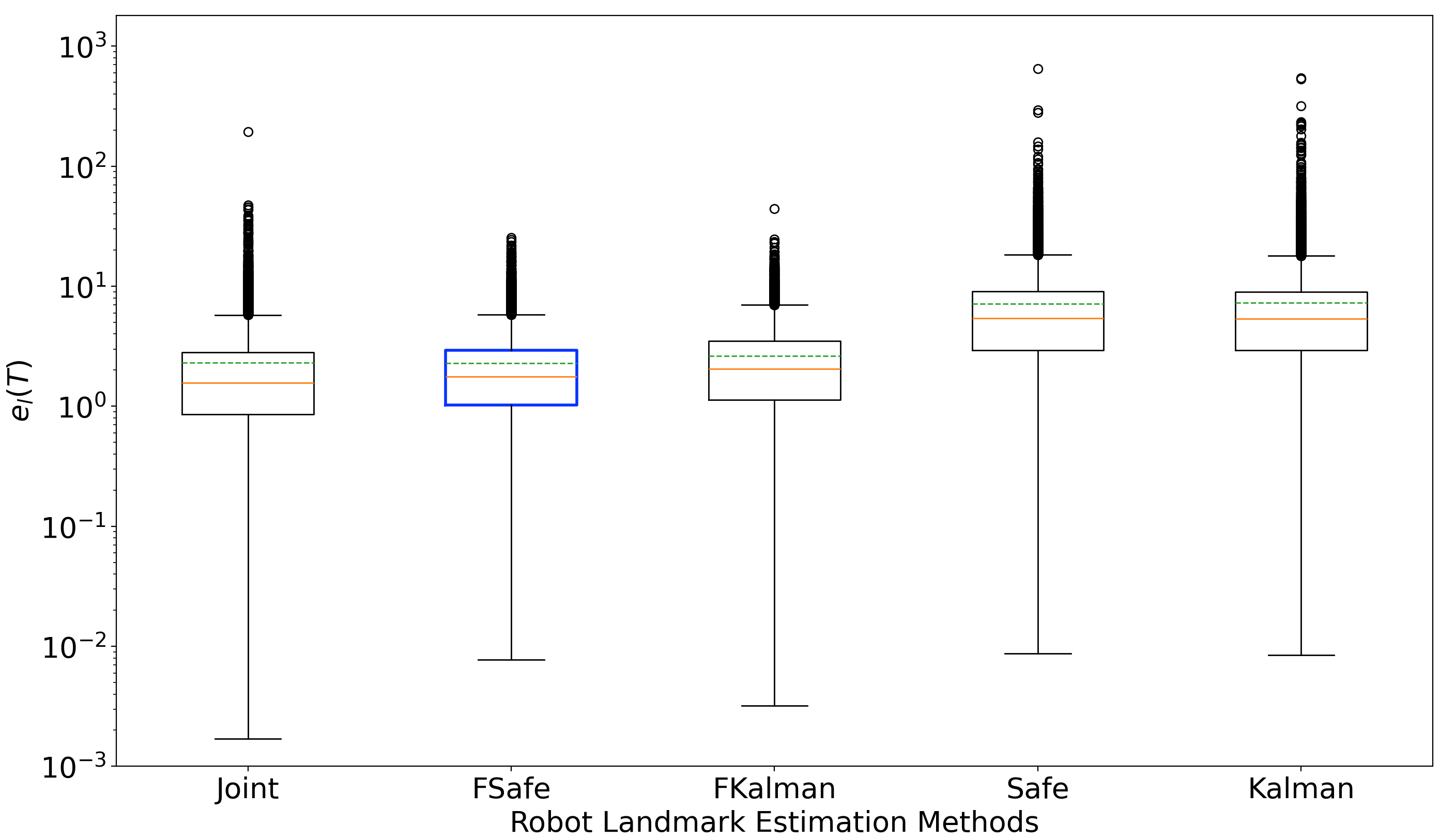}
    \caption{Box plot of the distribution of landmark estimation errors $e_l(T)$ at the end of each experiment. Note that the y-axis is on a logarithmic scale. The box associated with the proposed method, FSafe, is highlighted in blue. The orange line in each box shows the median of the distribution while the box ranges from first to third quartile with whiskers extending to 1.5 times the Inter-Quartile Range (IQR) beyond this. Outliers are shown as small circles. The mean of the distribution is shown as a dashed green line.}
    \label{fig:box}
\end{figure}

\begin{table}
    \centering
    \begin{tabular}{|c   c   c|}
    \hline
        Method & Mean & Std\\
        \hline
        \hline
        Joint & $2.298$ & $  2.853$\\
        \hline
        \textbf{FSafe} & ${\bf 2.275}$ & ${\bf 1.925}$\\
        \hline
        FKalman & $2.637$ & $2.186$\\
        \hline
        Safe & $7.163$ & $8.884$\\
        \hline
        Kalman & $7.32$ & $10.483$\\
        \hline
    \end{tabular}
    \caption{Mean and standard deviation of the distribution of $e_l(T)$ across the $20000$ randomised experiments (including outliers)}
    \label{tab:mstd}
\end{table}
As expected, the `Joint' method has the best overall inlier performance while in terms of outliers it shows higher errors compared to the proposed method thus leading to slightly higher values for the mean and standard deviation of the error distribution (see Table~\ref{tab:mstd} on the following page). These outliers may be due to numerical issues with the larger matrix computations in the `Joint' method. Nevertheless, the 'Joint' method does not enjoy the modularity properties that we have argued for and mostly serves as a comparison baseline for this study. The proposed method `FSafe' has the next best inlier performance while being modular and producing less outliers, thus leading to better performance in terms of mean and standard deviation of the error distribution. Also interesting to notice is that additional savings on computation (`FKalman'), communication (`Safe') or both (`Kalman') are achievable with graceful degradation of performance, at least for the scenario considered in this study. 

 Note that despite the extreme noise regimes considered, our nonlinear least squares bearing based robot-landmark localisation algorithm is very robust and on average achieves a landmark localisation accuracy of close to $2.3$ meters from initial estimation errors of up to $20$ meters. Importantly, the modular approach almost retains the performance of the full dimensional joint estimation algorithm and even produces less outliers, which points to the separation of concerns benefits of the modular approach. 

\section{Conclusions}\label{sec:conclusions}
In this paper we proposed a modular nonlinear least squares filtering methodology. We argued that modularity is desirable for managing complexity (such as the size of state vector representations and error covariance matrices), separation of concerns (if one filter diverges the others would not be immediately affected) and from systems engineering and maintenance perspectives (we could swap one of the two filters for an alternative without changing the other filter). We proposed a methodology that is safe against loops in the information flow graph that appear in the modular approach (in the sense of providing conservative estimates for the error covariance \cite{julierUhlmann97CI}) . We have also provided a derivation of this approach specifically for the 2D robot landmark locatisation problem using relative bearing measurements. Variants of this problem appear in many important robotic applications. We also discussed three variants of our approach that by relaxing requirements on communication, computation or both offer graceful degradation in performance. We provide a large randomised simulation study that demonstrates these ideas and validates our claims in practice.

\bibliographystyle{IEEEtran}
\bibliography{ref}
\end{document}